\documentclass[twoside]{article}

\usepackage[accepted]{aistats2020}

\runningtitle{A Practical Algorithm for Multiplayer Bandits when Arm Means Vary Among Players}

\runningauthor{Boursier, Kaufmann, Mehrabian, and Perchet}

\usepackage{natbib}
\usepackage{bm}
\usepackage{hyperref}       
\usepackage{url}            
\usepackage{booktabs}       
\usepackage{amsfonts}       
\usepackage{nicefrac}       
\usepackage{microtype}      
\usepackage{changepage}
\usepackage{macrosArticle,mathtools,afterpage}

\newcommand{\NameSecond}{M-ETC-Elim}
\newcommand{\eps}{\varepsilon}
\renewcommand{\epsilon}{\varepsilon}
\newcommand{\ex}[1]{ \bE\left[ #1 \right]}
\newcommand{\pr}[1]{ \bP \left ( #1 \right) }

\renewcommand{\tilde}[1]{\widetilde{#1}}
\renewcommand{\hat}[1]{\widehat{#1}}

\begin{document}

\twocolumn[
\aistatstitle{A Practical Algorithm for Multiplayer Bandits \\when Arm Means Vary Among Players}

\aistatsauthor{Etienne Boursier \And Emilie Kaufmann \And Abbas Mehrabian \And Vianney Perchet}

\newcommand{\authorsize}{\tiny}
\newcommand{\emailsize}{\bf \scriptsize}

\aistatsaddress{\scriptsize \parbox{5cm}{\centering Université Paris-Saclay, ENS Paris-Saclay \\ CNRS, Centre Borelli, Cachan, France} \\ \emailsize\texttt{eboursie@ens-paris-saclay.fr} \And  \scriptsize \parbox{5cm}{\centering Univ. Lille, CNRS, Inria SequeL, \\ UMR 9189 - CRIStAL, Lille, France} \\ \emailsize \texttt{emilie.kaufmann@univ-lille.fr} \And \scriptsize
	McGill University, Montr\'eal, Canada \\\emailsize \texttt{abbas.mehrabian@gmail.com} \And\scriptsize \parbox{5cm}{\centering CREST, ENSAE Paris, Palaiseau, France \\ Criteo AI Lab, Paris, France} \\\emailsize \texttt{vianney.perchet@normalesup.org}} 
]

\begin{abstract}
We study a multiplayer stochastic multi-armed bandit problem
in which players cannot communicate,
and if two or more players pull the same arm, a collision occurs and the involved players receive zero reward. 
We consider the challenging \emph{heterogeneous} setting, in which different arms may have different means for different players, and propose a new and efficient algorithm that combines the idea of leveraging forced collisions for implicit communication and that of performing matching eliminations. 
We present a finite-time analysis of our algorithm, 
giving the first sublinear minimax regret bound for this problem,
and prove that if the optimal assignment of players to arms is unique, our algorithm attains the optimal $O(\ln(T))$ regret, solving an open question raised at NeurIPS 2018 by \cite{got}.
\end{abstract}

\section{Introduction}

Stochastic multi-armed bandit models have been studied extensively  as they capture many sequential decision-making problems of practical interest. In the simplest setup, an agent repeatedly chooses among several actions (referred to as ``arms'') in each round of a game. To each action $i$ is associated a real-valued parameter $\mu_i$. Whenever the player performs the $i$th action (``pulls arm $i$''), she receives a random reward with mean $\mu_i$. The player's objective is to maximize the sum of rewards obtained during the game. If she knew the means associated with the actions before starting the game, 
she would play an action with the largest mean reward during all rounds.
The problem is to design a strategy for the player to maximize her reward in the setting where the means are unknown. The \emph{regret} of the strategy is the difference between the accumulated rewards in the two scenarios.

To minimize the regret, the player is faced with an exploration/exploitation trade-off as she should try (explore) all actions to estimate their means accurately enough but she may want to exploit the action that looks {\em probably} best given her current information. We refer the reader to~\citep{Bubeck:Survey12,torcsaba_book} for surveys on this problem. Multi-armed bandit (MAB) has been first studied as a simple model for sequential clinical trials~\citep{Thompson33,Robbins52Freq} but has also found many modern applications to online content optimization, such as the design of recommender systems~\citep{Li10contextual}. Recently, MAB algorithms have also been investigated for cognitive radios   \citep{Jouini09,anandkumar2011distributed}. In this context, arms model the available radio channels on which radio devices can communicate, and the reward associated with each arm is either a binary indicator of the success of a communication on that channel or some measure of its quality. 

The   applications to cognitive radios have motivated the \emph{multiplayer} bandit problem, in which several agents (devices) play on the same bandit (communicate using the same channels). If two or more agents pull the same arm, a \emph{collision} occurs and all agents pulling that arm receive zero reward. Without communicating, each agent must adopt a strategy aimed at maximizing the global reward obtained by all agents---so, we are considering a cooperative scenario rather than a competitive one. 
While most previous work on this problem focuses on the case in which the means of the arms are identical across players (the homogeneous variant), in this paper we study  the more challenging heterogeneous variant, in which each user may have a different utility for each arm: if player $m$ selects arm $k$, she receives a reward with mean $\mu_{k}^m$. This variant is more realistic for applications to cognitive radios, as the quality of each channel may vary from one user (device) to another, depending for instance on its configuration and location. 
%


More precisely, we study the model introduced by~\cite{got}, which has two main characteristics: first, each arm has a possibly different mean for each player; second, we are in a fully distributed setting with  no communication between players. Let $T$ denote the time horizon.
\citet{got} proposed an algorithm with regret bounded by $O((\ln T)^{2+\kappa})$ (for any constant $\kappa$), 
proved a lower bound of $\Omega(\ln T)$ \emph{for any algorithm}, and asked  if there is an algorithm matching this lower bound. In this paper, we propose a new algorithm for this model, \NameSecond{}, which depends on a hyperparameter $c$, and we upper bound its regret by $O(\ln(T)^{1+1/c})$ for any $c>1$. 
We also bound its worst-case regret by $O(\sqrt{T \ln T})$, which is the first sublinear minimax bound for this problem.
Moreover, if the optimal assignment of the players to the arms is unique, we prove that instantiating \NameSecond{} with $c=1$ yields regret at most $O(\ln(T))$, which is optimal and answers affirmatively the open question mentioned above in this particular case.\footnote{In practice, the optimal assignment may not be unique, but the players may circumvent this by adding a tiny random bias to their observations, independent of other players, and this will make the optimal assignment unique with high probability.} 
We present a non-asymptotic regret analysis of \NameSecond{} leading to nearly optimal regret upper bounds, and also demonstrate the empirical efficiency of this new algorithm via simulations.    

\paragraph{Outline} 
In Section~\ref{sec:Setup}, we formally introduce the heterogeneous multiplayer multi-armed bandit model and present our contributions. These results are put in perspective by comparison with the literature given in Section~\ref{sec:related}. We describe the \NameSecond{} algorithm in Section~\ref{sec:init} and upper bound its regret in Section~\ref{sec:ProofMain}. Finally, we report in Section~\ref{sec:experiments} results from an experimental study demonstrating the competitive practical performance of \NameSecond{}.

\section{Model and Contributions}\label{sec:Setup}
\label{sec:setup}

We study a multi-armed bandit model where $M$ players compete over $K$ arms, with $M \leq K$. We denote by $\mu_k^m$ the mean reward (or expected utility) of arm $k$ for player $m$. In each round $t = 1,2, \dots,T$, player $m$ selects arm $A^m(t)$ and receives a reward
\[R^m(t) = Y^m_{A^m(t),t} \left(1-\ind\left(\cC_{A^m(t),t}\right)\right),\]
where $(Y^m_{k,t})_{t=1}^{\infty}$ is an i.i.d.\ sequence with mean $\mu_{k}^m$ taking values in $[0,1]$,  $\cC_{k,t}$ is the event that at least two players have chosen arm $k$ in round $t$ (i.e., a collision occurs), and $\ind\left(\cC_{k,t}\right)$ is the corresponding indicator function. In the cognitive radio context, $Y^m_{k,t}$ is the quality of channel $k$ for player $m$ if she were to use this channel in isolation in round $t$, but her actual reward is zero if a collision occurs.

We assume that each player $m$ in  each round $t$ observes her reward $R^m(t)$ and the collision indicator $\ind\left(\cC_{A^m(t),t}\right)$. Note that in the special case in which the reward distributions satisfy $\bP(Y_{k,t}^m = 0)=0$ (e.g., if the corresponding distribution is continuous), the collision indicator $\ind\left(\cC_{A^m(t),t}\right)$ can  be reconstructed from the observation of $R^m(t)$. 
The decision of player $m$ at round $t$ can depend only on her past observations; that is, $A^m(t)$ is $\cF^m_{t-1}$ measurable, where $\cF^m_t = \sigma(A^m(1),R^m(1), \ind\!\left(\cC_{A^m(1),1}\right), \!\!\!\dots, A^m(t),R^m(t),$ $\ind\! \left(\cC_{A^m(t),t}\right))$. 

Hence, our setting is fully distributed: a player cannot use extra information such as observations made by others to make her decisions. Under this constraint, we aim at maximizing the global reward collected by all players. If the mean rewards $\mu_k^m$ were known and a central controller would assign arms to players, this would boil down to finding a maximum matching between players and arms.

A \emph{matching} is a one-to-one assignment of players to arms; 
formally, any one-to-one function
$\pi : [M] \to [K]$ is a matching, where we use the shorthand $[n] \coloneqq \{1,\dots,n\}$ for any integer $n$.  The \emph{utility} (or \emph{weight}) of a matching $\pi$ is defined as
$ U(\pi)\coloneqq \sum_{m=1}^M \mu_{\pi(m)}^m$.
We denote by $\cM$ the set of all matchings and let
$U^\star \coloneqq \max_{\pi \in \cM} U(\pi)$ denote the maximum attainable utility.
A {\em maximum matching} (or {\em optimal matching)} is a matching with utility $U^\star$.
The strategy maximizing the social utility of the players (i.e. the sum of all their rewards) would be to play
in each round according to a maximum matching, and the {\em (expected) regret} with respect to that oracle is defined as
\[ R_T \coloneqq T U^\star - \bE\left[\sum_{t=1}^T\sum_{m=1}^M R^m(t)\right].\]

Our goal is to design a strategy (a sequence of arm pulls) for each player that minimizes the regret. 
Our regret  bounds will depend on the gap between the utility of the best matching and the utility of the matching with the second best utility, defined as $\Delta \coloneqq \inf_{\pi : \Delta(\pi) > 0} \ \Delta(\pi)$, where $\Delta(\pi)\coloneqq U^\star - U(\pi)$. Note that $\Delta > 0$ even in the presence of several optimal matchings.
In the degenerate case that $\Delta(\pi)=0$ for all matchings $\pi$, we define $\Delta\coloneqq\infty$.

\paragraph{Contributions} We propose an efficient algorithm for the heterogeneous multiplayer bandit problem achieving (quasi) logarithmic regret. The algorithm, called Multiplayer Explore-Then-Commit with matching Elimination (\NameSecond{}), is described in detail in Section~\ref{sec:init}. It combines the idea of exploiting collisions for implicit communication, initially proposed by \cite{BoursierPerchet18} for the homogeneous setting (which we have improved and adapted to our setting), with an efficient way to perform ``matching eliminations.''

\NameSecond{} consists of several epochs combining exploration and communication, and may end with an exploitation phase if a unique optimal matching has been found. The algorithm depends on a parameter $c$ controlling the epoch sizes and enjoys the following regret guarantees.



\begin{theorem}\label{thm:generalepochsizeShort}\label{thm:main}
	(a)
The \NameSecond{}  algorithm with parameter $c\in\{1,2,\dots\}$ satisfies
\begin{align*}
R_T &= 	O\left(MK \left(\frac{M^2\ln(T)}{\Delta}\right)^{1+1/c}\right) \textnormal{ if } \Delta\neq\infty\textnormal{, and }\\
R_T &=O\left( M^3K\log(K)\sqrt{\log T} + M^2K \log(T)^{1+1/c}  \right)
\end{align*}
if $\Delta=\infty$.

(b) If the maximum matching is unique, \NameSecond{} with  $c=1$ satisfies 
\begin{equation*}R_T = O\left(\frac{M^3K\ln(T)}{\Delta}\right).
\end{equation*}
(c) Regardless of whether the optimal matching is unique or not, \NameSecond{} with  $c=1$ satisfies the minimax regret bound
\begin{equation*}
R_T = O\left(M^{\frac{3}{2}}\sqrt{KT\ln(T)}\right).
\end{equation*}
\end{theorem}

We emphasize that we carry out a non-asymptotic analysis of \NameSecond{}. The regret bounds of Theorem~\ref{thm:main} are stated with the $O(\cdot)$ notation for the ease of presentation
and the hidden constants depend on the chosen parameter $c$ only.
In Theorems~\ref{thm:generalepochsize},~\ref{thm:cyclethroughmatchings} and~\ref{thm:minmaxbound} we provide the counterparts of these results with explicit constants. 

A consequence of part (a) is that for a fixed problem instance, for any (arbitrarily small) $\kappa$, there exists an algorithm (\NameSecond{} with parameter $c=\lceil1/\kappa\rceil$) with regret $R_T=O((\ln(T))^{1+\kappa})$. This quasi-logarithmic regret rate improves upon the $O(\ln^2(T))$ regret rate of \cite{got}. Moreover, we provide additional theoretical guarantees for \NameSecond{} using the parameter $c=1$: an improved analysis in the presence of a unique optimal matching, which yields logarithmic regret (part (b)); and a problem-independent $O(\sqrt{T \ln T})$ regret bound (part (c)), which supports the use of this particular parameter tuning regardless of whether the optimal matching is unique. 
This is the first sublinear minimax regret bound for this problem.

To summarize, we present a unified algorithm that can be used in the presence of either a unique or multiple optimal matchings and get a nearly logarithmic regret in both cases, almost matching the known logarithmic lower bound.
Moreover, our algorithm is easy to implement, performs well in practice 
and does not need problem-dependent hyperparameter tuning.

\section{Related Work}\label{sec:related}

\paragraph{Centralized Variant} Relaxing the decentralization assumption, i.e., when a central controller is jointly selecting $A^1(t), \dots,A^M(t)$, our problem coincides with a combinatorial bandit problem with semi-bandit feedback, which is studied by \cite{GaiKJ12}. More precisely, introducing $M\times K$ \emph{elementary arms} with means $\mu_{m}^k$ for $m \in [M]$ and $k \in [K]$, the central controller selects at each time-step $M$ elementary arms whose indices form a matching. Then, the reward of each chosen elementary arm is observed and the obtained reward is their sum. A well-known algorithm for this setting is CUCB \citep{Chen13Comb}, whose regret satisfies $R_T = O\left(({M^2 K}/{\Delta})\ln(T)\right)$ \citep{KvetonWAS15}. \citet{Wang18} also proposed a Thompson sampling-based algorithm with a similar regret bound. Improved dependency in $M$ 
was obtained for the ESCB algorithm  \citep{CombUCB15,DegenneP16}, which is less numerically appealing as it requires to compute an upper confidence bound for each matching in every round. 
In this work, we propose an efficient algorithm with regret upper bounded by (roughly) $O\left(({M^3 K}/{\Delta})\ln(T)\right)$ for the more challenging decentralized setting. 


\paragraph{Homogeneous Variant} Back to the decentralized setting, the particular case in which all players share a common utility for all arms, i.e. $\mu^m_k = \mu_k$ for all $m \in [M]$, has been studied extensively:
the first line of work on this variant combines standard bandit algorithms with an orthogonalization mechanism \citep{liuzhao,anandkumar2011distributed,emily_multiplayer}, and obtains logarithmic regret, with a large multiplicative constant due to the number of collisions. \cite{musicalchair} proposes an algorithm based on a uniform exploration phase in which each player identifies the top $M$ arms, followed by a ``musical chairs'' protocol that allows each player to end up at a different arm quickly. Drawing inspiration from this musical chairs protocol, \cite{BoursierPerchet18} recently proposed an algorithm with an $O\left(\left( (K-M)/\Delta+ KM\right)\ln(T)\right)$ regret bound, which relies on two other crucial ideas:  {exploiting collisions for communication} and {performing arm eliminations}. Our algorithm also leverages these two ideas, with the following enhancements. The main advantage of our communication protocol over that of~\cite{BoursierPerchet18} is  that the followers only send each piece of information once, to the leader, instead of sending it to the $M-1$ other players. 
Then, while \cite{BoursierPerchet18} uses \emph{arm eliminations} (coordinated between players) to reduce the regret, we cannot employ the same idea for our heterogeneous problem, as an arm that is bad for one player might be good for another player, and therefore cannot be eliminated. Our algorithm instead relies on \emph{matching eliminations}. 

\paragraph{Towards the Fully Distributed Heterogeneous Setting} 
Various semi-distributed variants of our problem in which some kind of communication is allowed between players have been studied by
\cite{Avner2016,Kalathil2014, Nayyar2018}.
In particular, the algorithms proposed by ~\cite{Kalathil2014, Nayyar2018} require a pre-determined channel dedicated to communications: in some phases of the algorithm, players in turn send information (sequences of bits) on this channel, and it is assumed that all other players can observe the sent information. 


The fully distributed setting was first studied by \cite{got}, who proposed the Game-of-Thrones (GoT) algorithm and proved its regret is bounded by $O((\ln T)^{2+\kappa})$ for any given constant $\kappa>0$, if its parameters are ``appropriately tuned.'' 
In a recent preprint \citep{got_arxiv}, the same authors provide an improved analysis, showing the same algorithm (with slightly modified phase lengths) enjoys quasi-logarithmic regret $O((\ln T)^{1+\kappa})$. GoT is quite different from \NameSecond{}: it proceeds in epochs, each consisting of an exploration phase, a so-called GoT phase and an exploitation phase. During the GoT phase, the players jointly run a Markov chain whose unique stochastically stable state corresponds to a maximum matching of the estimated means. 
A parameter $\eps \in (0,1)$ controls the accuracy of the estimated maximum matching obtained after a GoT phase. 
Letting $c_1,c_2,c_3$ be the constants parameterizing the lengths of the phases, the improved analysis of GoT~\citep{got_arxiv} upper bounds its regret by 
$M c_3 2^{k_0+1} +
2(c_1 + c_2) M \log_2^{1+\kappa}
\left({T}/{c_3}+2\right).$ 
This upper bound is asymptotic as it holds for $T$ \emph{large enough}, where ``how large'' is not explicitly specified and \emph{depends on $\Delta$}.\footnote{ \cite[Theorem~4]{got_arxiv} requires $T$ to be larger than $ c_3(2^{k_0} - 2)$, where $k_0$ satisfies Equation (16), which features $\kappa$ and $\Delta$.} Moreover, the upper bound is  valid only when the parameter $\eps$ is chosen \emph{small enough}: $\eps$ should satisfy some constraints (Equations (66)-(67)) also featuring $\Delta$. Hence, a valid tuning of the parameter $\eps$  would require prior knowledge of arm utilities. In contrast, we provide in Theorem~\ref{thm:generalepochsize} a non-asymptotic regret upper bound for \NameSecond{}, which holds for any choice of the parameter $c$ controlling the epoch lengths. Also, we show that if the optimal assignment is unique, \NameSecond{} has logarithmic regret.
Besides, we also illustrate in Section~\ref{sec:experiments} that \NameSecond{} outperforms GoT in practice.
Finally, GoT has several parameters to set ($\delta, \eps, c_1,c_2,c_3$), while \NameSecond{} has only one integral parameter $c$, and setting $c=1$ works very well in all our experiments.

If $\Delta$ is known, an algorithm with similar ideas to \NameSecond{} with $O(\log T)$ regret was presented independently in the recent preprint of~\cite{magesh}.

Finally, 
the recent independent preprint of~\cite{Tibrewal} 
studies a slightly stronger feedback model than ours: they assume each player in each round has the option of ``observing whether a given arm has been pulled by someone,'' without actually pulling that arm (thus avoiding collision due to this ``observation''), an operation that is called ``sensing.'' Due to the stronger feedback, communications do not need to be implicitly done through collisions and bits can be broadcast to other players via sensing. Note that it is actually possible to send a single bit of information from one player to all other players in a single round in their model, an action that requires $M-1$ rounds in our model. 
Still, the algorithms proposed by \cite{Tibrewal} can be modified to obtain algorithms for our setting, and \NameSecond{} can also be adapted to their setting. The two algorithms proposed by \cite{Tibrewal} share similarities with \NameSecond{}: they also have exploration, communication and exploitation phases, but they do not use eliminations. Regarding their theoretical guarantees, a first remark is that those proved in~\cite{Tibrewal} only hold in the presence of a unique optimal matching, whereas our analysis of \NameSecond{} applies in the general case. 
The second remark is that their regret bounds for the case in which $\Delta$ is unknown (Theorems~3(ii) and~4) feature exponential dependence on the gap $1/\Delta$, whereas our regret bounds have polynomial dependence. Finally, the first-order term of their Theorem~4 has a quadratic dependence in $1/\Delta$, whereas our Theorem~\ref{thm:main}(b) scales linearly,  which is optimal (see the lower bounds section below)
and allows us to get the $\widetilde{O}(\sqrt T)$ minimax regret bound for \NameSecond{}.

\paragraph{Lower Bounds} The $\Omega(\ln(T))$ lower bound proven by \cite{got} hides the problem parameters;
we next review the lower bounds that flesh out the dependence
on $K,M$ and $\Delta$. In the (easier) centralized setting discussed above, an asymptotic lower bound of  $\Omega((K-M)\ln(T)/\Delta)$ is proved in the homogeneous case \citep{Anantharam87}. In the centralized  heterogeneous case, \cite{CombUCB15} obtain a general problem dependent lower bound for combinatorial semi-bandits of the form $c(\bm\mu,M)\ln(T)$ and show that $c(\bm\mu,M) = \Theta\left(K/\Delta\right)$ for many common combinatorial structures, including matchings. A minimax lower bound of $\Omega(\sqrt{MKT})$ was given by \cite{AudibertBL14} in the same setting. These lower bounds show that the dependency in $T,\Delta$ and $K$ obtained in Theorem~\ref{thm:main}(b),(c) are essentially not improvable, but that the dependency in $M$ might be. However, this observation can be mitigated by noting that finding an algorithm whose regrets attain the available lower bounds for combinatorial semi-bandits is already hard even without the extra challenge of decentralization.

\section{The \NameSecond{} Algorithm}\label{sec:init}

Our algorithm relies on an initialization phase in which the players elect a leader in a distributed manner. Then a communication protocol is set up, in which the leader and the followers have different roles: followers explore some arms and communicate to the leader estimates of the arm means, while the leader maintains a list of ``candidate optimal matchings'' and communicates to the followers the list of arms that need exploration in order to refine the list, i.e. to eliminate some candidate matchings. The algorithm is called \emph{Multiplayer Explore-Then-Commit with matching Eliminations}
(\NameSecond{} for short). Formally, each player executes  Algorithm~\ref{algo:GeneAlgo} below.

\begin{algorithm}[h]
	\DontPrintSemicolon
	\caption{\NameSecond{} with parameter $c$}
	\label{algo:GeneAlgo}
	\KwIn{Time horizon $T$, number of arms $K$}
	$R,M\longleftarrow \textnormal{\sc Init}(K,1/KT)$\;
	\leIf{$R=1$}{\textsc{LeaderAlgorithm}(M)}{\textsc{FollowerAlgorithm}(R,M)}
\end{algorithm}

\NameSecond{} requires as input the number of arms $K$ (as well as a shared numbering of the arms across the players) and the time horizon $T$ (the total number of arm selections). However, if the players know only an upper bound on $T$, our results hold with $T$ replaced by that upper bound as well. If no upper bound on $T$ is known, the players can employ a simple doubling trick \citep{DoublingPaper}:
we execute the algorithm assuming $T=1$,
then we execute it assuming $T=2\times 1$,
and so on, until the actual time horizon is reached.
If the expected regret of the algorithm for a known time horizon $T$
is $R_T$, then the expected regret of the modified algorithm for unknown time horizon $T$ would be 
$R'_T \leq \sum_{i=0}^{\log_2(T)} R_{2^i} \leq \log_2(T) \times R_T.$


\paragraph{Initialization} The initialization procedure, borrowed from \cite{BoursierPerchet18}, outputs for each player a rank $R \in [M]$ as well as the value of $M$, which is initially unknown to the players. This initialization phase relies on a ``musical chairs'' phase after which the players end up on distinct arms, followed by a  ``sequential hopping'' protocol that permits them to know their ordering. For the sake of completeness, the initialization procedure is described in detail in Appendix~\ref{sec:Init}, where we also prove the following.


\begin{lemma}\label{lem:init}
	Fix $\delta_0>0$. With probability at least $1-\delta_0$, if the $M$ players run the $\textnormal{\sc Init}(K,\delta_0)$ procedure, which takes $K  \ln (K/\delta_0) + 2K-2< K \ln(e^2 K /\delta_0)$ many rounds, all players learn $M$ and obtain a distinct ranking from 1 to $M$. 
\end{lemma}


\paragraph{Communication Phases} Once all players have learned their ranks, player 1 becomes the \emph{leader} and other players become the \emph{followers}. The leader executes additional computations, and communicates with the followers  individually, while each follower communicates only with the leader. 

The leader and follower algorithms, described below, rely on several \emph{communication phases}, which start at the same time for every player. During communication phases, the default behavior of each player is to pull her \emph{communication arm}. It is crucial that these communication arms are distinct: an optimal way to do so is for each player to use her arm in the best matching found so far. In the first communication phase, such an assignment is unknown and players simply use their ranking as communication arm. 
Suppose at a certain time the leader wants to send a sequence of $b$ bits $t_1,\dots,t_b$ to the player with ranking $i$ and communication arm $k_i$.
During the next $b$ rounds, for each $j=1,2,\dots, b$, if $t_j=1$, the leader pulls arm $k_i$; otherwise, she pulls her own communication arm $k_1$, while all followers stick to their communication arms. Player $i$ can thus reconstruct these $b$ bits after these $b$ rounds, by observing the collisions on arm $k_i$. 
The converse communication between follower $i$ and the leader is similar.
The rankings are also useful to know \emph{in which order communications should be performed}, as the leader successively communicates messages to the $M-1$ followers, and then the $M-1$ followers successively communicate messages to the leader. 

Note that in case of unreliable channels where some of the communicated bits may be lost, there are several options to make this communication protocol more robust, such as sending each bit multiple times or using the Bernoulli signaling protocol of \cite{Tibrewal}. Robustness has not been the focus of our work.


\paragraph{Leader and Follower Algorithms} 
The leader and the followers perform distinct algorithms, explained next.
Consider a bipartite graph
with parts of size $M$ and $K$, where the edge $(m,k)$  has weight  $\mu_k^m$ and associates player $m$ to arm $k$.
The weights $\mu_k^m$ are unknown to the players, but the leader maintains a set of \emph{estimated}  weights that are sent to her by the followers, and approximate the real weights. 
The goal of these algorithms is for the players to jointly explore the matchings in this graph, while gradually focusing on better and better matchings. 
For this purpose, the leader maintains a set of {\em candidate edges} $\cE$, which is initially $[M]\!\times\![K]$, that can be seen as edges that are potentially contained in optimal matchings, and gradually refines this set by performing eliminations, based on the information obtained from the exploration phases and shared during communication phases.

\NameSecond{} proceeds in epochs whose length is parameterized by $c$. In epoch $p=1,2,\dots$, the leader weights the edges using the estimated weights.
Then for every edge $(m,k) \in \cE$, the leader computes the associated matching $\widetilde{\pi}_p^{m,k}$ defined as the estimated maximum matching containing the edge $(m,k)$. This computation can be done in polynomial time using, e.g., the Hungarian algorithm \citep{hungarian}. The leader then computes the utility of the maximum matching and eliminates from $\cE$ any edge for which the weight of its associated matching is smaller by at least $4 M \eps_p$, where 
\begin{small}
\begin{equation}
\eps_p \coloneqq \sqrt{\frac{\ln (2/\delta)}{2^{1+p^c}}},
\text{ with } \delta \coloneqq \frac{1}{M^2KT^2}.
\label{def:deltaeps2new}
\end{equation}
\end{small}%
The leader then forms the set of associated candidate matchings $\cC \coloneqq\{\widetilde{\pi}_p^{m,k} , (m,k) \in \cE\}$   and communicates to each follower the list of arms to explore in these matchings. 
Then exploration begins, in which for each candidate matching every player pulls its assigned arm $2^{p^c}$ times and records the received reward. Then another communication phase begins, during which each follower sends her observed estimated mean for the arms to the leader. More precisely, for each explored arm, the follower truncates the estimated mean (a number in $[0,1]$) and sends only the $\frac{p^c+1}{2}$ most significant bits of this number to the leader.
The leader updates the estimated weights and  everyone proceeds to the next epoch. If at some point the list of candidate matchings $\cC$ becomes a singleton, it means that (with high probability) the actual maximum matching is unique and has been found; so all players  jointly pull that matching for the rest of the game (the exploitation phase).

\paragraph{Possible Exploitation Phase} Note that in the presence of several optimal matchings, the players will not enter the exploitation phase but will keep exploring several optimal matchings, which still ensures small regret. On the contrary, in the presence of a unique optimal matching, they are guaranteed to eventually enter the exploitation phase.\footnote{This different behavior is the main reason for the improved regret upper bound obtained when the optimal matching is unique.} Also, observe that the set $\cC$ of candidate optimal matchings does not necessarily contain \emph{all} potentially optimal matchings, but all the edges in those matchings remain in $\cE$ and are guaranteed to be explored.

The pseudocode for the leader's algorithm is given below, while the corresponding follower algorithm appears in Appendix~\ref{sec:Init}. In the pseudocodes, \texttt{(comm.)} refers to a call to the communication protocol.

	\begin{procedure}[h]\begin{small}
		\DontPrintSemicolon
		\caption{LeaderAlgorithm(M) for the \NameSecond{} algorithm with parameter $c$}
		\KwIn{Number of players $M$}
		$\cE \longleftarrow [M]\times[K]$\tcp*[r]{list of candidate edges}
		$\widetilde{\mu}^m_{k} \longleftarrow 0$ for all $(m,k)\in[M]\times[K]$\tcp*[r]{empirical estimates for utilities}
		\For{$p=1,2,\dots$} {
			$\cC \longleftarrow \emptyset$\tcp*[r]{list of associated matchings}
			$\pi_1 \longleftarrow \argmax{} \left\{ \sum_{n=1}^{M} \widetilde{\mu}^n_{\pi(n)} : \pi \in \cM
			\right\}$\tcp*[r]{using Hungarian algorithm}
			\For{$(m,k) \in \cE$}
			{	$\pi \longleftarrow \argmax{} \left\{ \sum_{n=1}^{M} \widetilde{\mu}^n_{\pi(n)} : \pi(m)=k \right\}$\tcp*[r]{using Hungarian algorithm}
				\lIf{$\sum_{n=1}^{M}\!\! \left\{\widetilde{\mu}^n_{\pi_1(n)
					}-\widetilde{\mu}^n_{\pi(n)
					}\!\right\}\! \leq \! 4M\epsilon_p$\!\!
					\label{gapdefined}}
				{add $\pi$ to $\cC$}
				\lElse{remove $(m,k)$ from $\cE$}
			}
			\For{each player $m=2,\dots,M$}{
				Send to player $m$ the value of $\operatorname{size}(\cC)$\tcp*[f]{(comm.)}
				\For{$i=1,2,\dots,\operatorname{size}(\cC)$} {
					Send to player $m$ the arm associated to player $m$ in $\cC[i]$
					\tcp*[r]{(comm.)}
				}
				Send to player $m$ the communication arm of the leader and player $m$, namely $\pi_1(1)$ and $\pi_1(m)$
			}
			\lIf{$size(\cC)=1$} {pull for the rest of the game the arm assigned to player $1$ in the unique matching in $\cC$\tcp*[f]{enter the exploitation phase}}
			
			
			\For{$i=1,2,\dots,\operatorname{size}(\cC)$} {
				pull $2^{p^c}$ times the arm assigned to player $1$ in the matching $\cC[i]$\tcp*[r]{exploration} 
			}
			\For{$k=1,2,\dots,K$}{
				$\widetilde{\mu}_k^1\longleftarrow$ empirically estimated utility of arm $k$ if it was pulled in this epoch, 0 otherwise\;
			}
			Receive the values $\widetilde{\mu}_1^m,\widetilde{\mu}_2^m,\dots,\widetilde{\mu}_K^m$ from each player $m$\tcp*[r]{(comm.)}
		}\end{small}
	\end{procedure}


\section{Analysis of \NameSecond{}}\label{sec:ProofMain}
We may assume that $K\leq T$, otherwise all parts of Theorem~\ref{thm:main} would be trivial, since $R_T\leq MT$ always.
Theorem~\ref{thm:generalepochsize} provides a non-asymptotic upper bound on the regret of \NameSecond{} when $\Delta\neq\infty$. 
 
\begin{theorem}\label{thm:generalepochsize} Let $\pi^{m,k}$ be the best suboptimal matching assigning arm $k$ to player $m$, namely,
\[\pi^{m,k}\in \argmax{} \left\{ U(\pi): \pi(m)=k \text{ and } U(\pi)<U^\star\right\}.\]
For all $c \geq 1$, let $T_0(c):=\exp\left(2^{\frac{c^c}{\ln^c(1+\frac{1}{2c})}}\right)$. For all $T \geq T_0(c)$, if $\Delta\neq\infty$, the regret of \NameSecond{} with parameter $c$ is upper bounded as\footnote{In this paper, $\ln(\cdot)$ and $\lg(\cdot)$ denote the natural logarithm and the logarithm in base 2, respectively.}
\begin{small}
\begin{align*}
R_T &\leq 2+MK \ln (e^2 K^2 T)
+  6 M^2 K \lg(K) (\lg T)^{1/c}\\
&+  e^2 M K  (\lg T)^{1+1/c} + \frac{2M^3K\lg (K)}{\sqrt{2}-1}\sqrt{\ln(2M^2 K T^2)} \\ &+  \frac{2 \sqrt{2}}{3-2\sqrt{2}} M^2K \sqrt{\ln(2M^2KT^2)} \lg(\ln(T))\\
& + \frac{2\sqrt{2}-1}{\sqrt{2}-1} \!\!\!\!\!\mathlarger{\mathlarger{\sum_{(m,k) \in [M] \times [K]}} }\!\!\!\!\!
\left(\!\frac{32 M^2 \ln(2M^2KT^2)}{\Delta(\pi^{m,k})}\!\right)^{1+1/c}\!\!.
\end{align*}\end{small}
\end{theorem}
%
%
The first statement of Theorem~\ref{thm:generalepochsizeShort}(a) easily follows by lower bounding 
$\Delta(\pi^{m,k})\geq \Delta$ for all $m,k$. 
The second statement is proved by noting that if $\Delta=\infty$, then the exploration phases incur zero regret, so in that case a variant of Theorem~\ref{thm:generalepochsize} holds without the last term on the right-hand-side.
Parts (b) and (c) of Theorem~\ref{thm:generalepochsizeShort} similarly follow respectively from Theorems~\ref{thm:cyclethroughmatchings} and~\ref{thm:minmaxbound} in Appendices~\ref{app:uniquematching} and~\ref{app:minmaxbound}, with proofs similar to that of Theorem~\ref{thm:generalepochsize} presented below.

The constant $T_0(c)$ in Theorem~\ref{thm:generalepochsize} equals 252 for $c=1$ but becomes large when $c$ increases. Still, the condition on $T$ is explicit and independent of the problem parameters. In the case of multiple optimal matchings, our contribution is mostly theoretical, as we would need a large enough value of $c$ and a long time $T_0 (c)$ for reaching a prescribed $\ln^{ 1+o(1)} (T)$ regret. However, in the case of a unique optimal matching (common in practice, and sometimes assumed in other papers), for the choice $c = 1$, the logarithmic regret upper bound stated in Theorem~\ref{thm:cyclethroughmatchings} is valid for all $T\geq1$. Even if there are several optimal matchings, the minimax bound of Theorem~\ref{thm:minmaxbound} gives an $O(\sqrt{T \ln T})$ regret bound that is a best-possible worst-case bound (also known as the minimax rate), up to the $\sqrt {\ln T}$ factor.  Hence M-ETC-Elim with $c = 1$ is particularly good, both in theory and in practice. Our experiments also confirm that for $c = 1,2$ the algorithm performs well (i.e., beats our competitors) even in the presence of multiple optimal matchings.

\subsection{Sketch of Proof of Theorem~\ref{thm:generalepochsize}}
The analysis relies on several lemmas with proofs delayed to Appendix~\ref{app:elimproof}.
Let $\cC_p$ denote the set of candidate matchings used in epoch $p$, and for each matching $\pi$ let $\tilde{U}_p(\pi)$ be the utility of $\pi$ that the leader can estimate based on the information received by the end of epoch $p$. Let $\hat{p}_T$ be the total number of epochs before the (possible) start of the exploitation phase. As $2^{\hat{p}_T^c} \leq T$, we have $\hat{p}_T \leq \lg(T)$. Recall that a successful initialization means all players identify $M$ and their ranks are  distinct. Define the \emph{good event}
\begin{align} 
	& \cG_T  \coloneqq \Big\{ \textsc{Init}(K,1/KT) \text{ is successful and }\notag \\ 
	&\forall p \leq \hat{p}_T, \forall \pi \in \cC_{p+1}, |\tilde{U}_p(\pi) - U(\pi)|\leq 2M\epsilon_p\Big\}.\label{def:GoodEvent}
\end{align}
During epoch $p$, for each candidate edge $(m,k)$, player $m$ has pulled arm $k$ at least $2^{p^c}$ times and the quantization error is smaller than $\epsilon_p$. Hoeffding's inequality and a union bound over at most $\lg(T)$ epochs (see Appendix~\ref{sec:conc}) together with Lemma~\ref{lem:init} yield that $\cG_T$ holds with large probability.

\begin{lemma}\label{lem:conc}
$\bP\left(\cG_T\right) \geq 1 - \frac{2}{MT}$.
\end{lemma} 

If $\cG_T$ does not hold, we may upper bound the regret by $MT$. Hence it suffices to bound the expected regret conditional on $\cG_T$, and the unconditional expected regret is bounded by this value plus 2.

Suppose that $\cG_T$ happens. First, the regret incurred during the initialization phase is upper bounded by $MK \ln (e^2K^2T)$ by Lemma~\ref{lem:init}. 
Moreover, the gap between the best estimated matching of the previous phase and the best matching is at most $2M \epsilon_{p-1}$ during epoch $p$. Any single communication round then incurs regret at most $2 + 2M\epsilon_{p-1}$, the first term being due to the collision between the leader and a follower, the second to the gap between the optimal matching and the matching used for communication. Summing over all communication rounds and epochs leads to Lemma~\ref{lem:initcomm} below.
\begin{lemma}\label{lem:initcomm}
The regret due to communication is bounded by
\begin{align*}
	3 & M^2K \lg (K) \hat{p}_T + \frac{2^c\sqrt{2}}{3-2\sqrt{2}} M^2 K\sqrt{\ln(2/\delta)}\\
	&+M K (\hat{p}_T)^{c+1}+ \frac{2M^3K\lg (K)}{\sqrt{2} -1} \sqrt{\ln(2/\delta)} . \end{align*} 
\end{lemma}
%
%
For large horizons, Lemma~\ref{lem:Tassumption} bounds some terms such as $\hat p_T$ and $(\hat p_T)^c$. When $c=1$, tighter bounds that are valid for any $T$ are used to prove Theorems~\ref{thm:generalepochsizeShort}(b) and \ref{thm:generalepochsizeShort}(c).

\begin{lemma}\label{lem:Tassumption} For any suboptimal matching $\pi$, let  \(P(\pi) : = \inf \{ p \in \N : 8 M \eps_{p} < \Delta(\pi)\}.\) The assumption $T\geq T_0(c)$ implies that for any matching~$\pi$,
$\Delta(\pi) 2^{P(\pi)^c} \leq \left( \frac{32M^2 \ln(2M^2KT^2)}{\Delta(\pi)}\right)^{1+\frac{1}{c}}$. Also, $2^c \leq 2\lg(\ln(T))$, $\hat p_T \leq 2 (\lg T)^{1/c}$ and $(\hat p_T)^c \leq e \lg T$.
\end{lemma}

Hence for $T\geq T_0(c)$, we can further upper bound the first three terms of the sum in Lemma~\ref{lem:initcomm} by
\begin{small}
\begin{align}
&6 M^2 K \lg(K) (\lg T)^{1/c}
+  e^2 M K  (\lg T)^{1+1/c} \notag\\&
+ \frac{2\sqrt{2}}{3-2\sqrt{2}}M^2K \sqrt{\ln(2/\delta)} \lg(\ln(T)).\label{nonunique}
\end{align}\end{small}It then remains to upper bound the regret incurred during exploration and exploitation phases. On $\cG_T$, during the exploitation phase the players are jointly pulling an optimal matching and no regret is incurred.
For an edge $(m,k)$, let $\widetilde{\Delta}_p^{m,k}\coloneqq U^\star - U(\widetilde{\pi}^{m,k}_p)$ be the gap of its associated matching at epoch $p$. During any epoch $p$, the incurred regret is then $\sum_{\pi \in \cC_p} \Delta(\pi) 2^{p^c} = \sum_{(m,k) \in \cE} \widetilde{\Delta}_p^{m,k} 2^{p^c}$.
	
Recall that $\pi^{m,k}$ is the best suboptimal matching assigning arm $k$ to player $m$.
Observe that for any epoch $p > P(\pi^{m,k})$, since $\cG_T$ happens, $\pi^{m,k}$ (and any worse matching) is not added to $\cC_p$; thus during any epoch $p> P(\pi^{m,k})$, the edge $(m,k)$ is either eliminated from the set of candidate edges, or it is contained in some optimal matching and satisfies $\widetilde{\Delta}_p^{m,k}=0$. 
	Hence, the total regret incurred during exploration phases is bounded by
	\begin{small}	 \begin{equation} \label{eq:exploregret}
	\sum_{(m,k) \in [M]\times[K]}\sum_{p=1}^{P(\pi^{m,k})} \widetilde{\Delta}_p^{m,k} 2^{p^c}.
	\end{equation}\end{small}The difficulty for bounding this sum is that  $\widetilde{\Delta}_p^{m,k}$  is random since $\tilde{\pi}_p^{m,k}$ is random.
However, 
$\widetilde{\Delta}_p^{m,k}$ can be related to $\Delta(\pi^{m,k})$ by $\widetilde{\Delta}_p^{m,k} \leq \frac{\epsilon_{p-1}}{\epsilon_{P(\pi^{m,k})}} \Delta(\pi^{m,k})$.
A convexity argument then allows us to bound the ratio $\frac{\epsilon_{p-1}}{\epsilon_{P(\pi^{m,k})}}$, which yields Lemma~\ref{lemma:exploelimm}, proved in Appendix~\ref{proof:exploelimm}.
\begin{lemma}\label{lemma:exploelimm}
For any edge $(m,k)$, if $p < P(\pi^{m,k})$ then 
$
\widetilde{\Delta}_p^{m,k} 2^{p^c} \leq \Delta(\pi^{m,k}) \frac{2^{P(\pi^{m,k})^c}}{\sqrt{2}^{P(\pi^{m,k})-(p+1)}}.
$
\end{lemma}

\begin{figure*}
	\begin{center}
		\hspace{-0.5cm}
		\begin{tabular} {c c}	
			\includegraphics[scale=0.85]{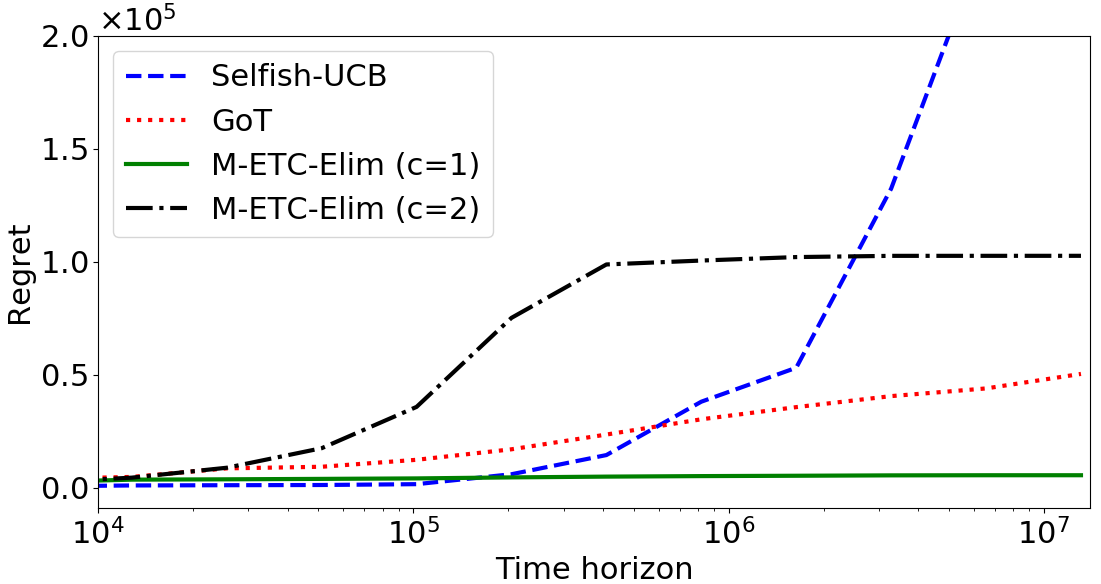}
			&
			\includegraphics[scale=0.85]{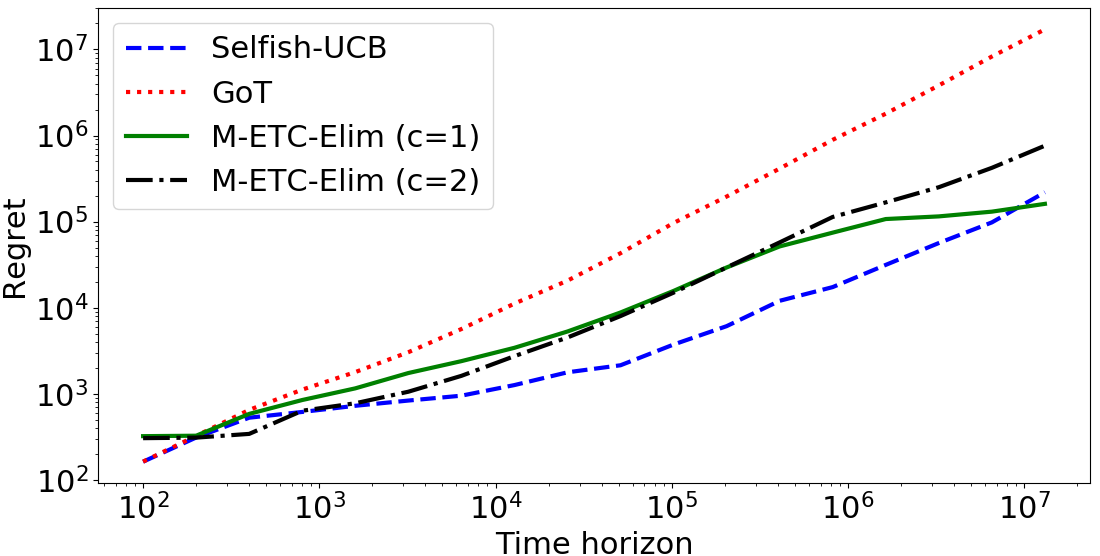}
		\end{tabular}
	\end{center}
	\caption{$R_T$ as a function of $T$ with reward matrices $U_1$ (left)
		and $U_2$ (right) and Bernoulli rewards. 
	}
	\label{expmain}	
\end{figure*}

By Lemma~\ref{lemma:exploelimm}, $\sum_{p=1}^{P(\pi^{m,k})} \widetilde{\Delta}_p^{m,k} 2^{p^c}$ is upper bounded by $
\Big({{\sum_{p=0}^{\infty}} } 1/\sqrt{2}^{p} \Big) \Delta(\pi^{m,k}) 2^{P(\pi^{m,k})^c} + \widetilde{\Delta}^{m,k}_{P(\pi^{m,k})} 2^{P(\pi^{m,k})^c}$.
As $\widetilde{\pi}^{m,k}_{P(\pi^{m,k})}$ is either optimal or its gap is larger than $\Delta(\pi^{m,k})$, Lemma~\ref{lem:Tassumption} yields \begin{small}\[\widetilde{\Delta}^{m,k}_{P(\pi^{m,k})} 2^{P(\pi^{m,k})^c} \leq \left(
\frac{32 M^2 \ln(2M^2KT^2)}{\Delta(\pi^{m,k})} \right)^{1+1/c}\] \end{small} in both cases. Therefore, we find that
\begin{small}
\begin{align*}
{{\sum_{p=1}^{P(\pi^{m,k})}} } \widetilde{\Delta}_p^{m,k} 2^{p^c} 
\leq \frac{2\sqrt{2}-1}{\sqrt{2}-1} \left(\frac{32 M^2 \ln(2M^2KT^2)}{\Delta(\pi^{m,k})} \right)^{1+1/c}.
\end{align*}
\end{small}Plugging this bound in \eqref{eq:exploregret}, the bound \eqref{nonunique} in Lemma~\ref{lem:initcomm} and summing up all terms yields Theorem~\ref{thm:generalepochsize}.

\subsection{Proof of Theorem~\ref{thm:generalepochsizeShort}(b), Unique Optimal Matching }
The reader may wonder why can we obtain a better (logarithmic) bound if the maximum matching is unique. The intuition is as follows: in the presence of a unique optimal matching, \NameSecond\ eventually enters the exploitation phase (which does not happen with multiple optimal matchings), and we can therefore provide a tighter bound on the number of epochs before exploitation phase compared with the one provided by Lemma~\ref{lem:Tassumption}. More precisely, in that case 
we have
\begin{small}\( \hat{p}_T \leq \lg \left( 64M^2 \Delta^{-2}\ln(2M^2KT^2)\right).\)\end{small}
Moreover, another bound given by Lemma~\ref{lem:Tassumption} can be tightened when $c=1$ regardless of whether the optimal matching is unique or not:
\begin{small}
\(\Delta(\pi) 2^{P(\pi)} \leq 64M^2 \ln(2M^2KT^2)/\Delta(\pi).\)\end{small}These two inequalities lead to Theorem~\ref{thm:generalepochsizeShort}(b), proved in Appendix~\ref{app:uniquematching}.

\subsection{Proof of Theorem~\ref{thm:generalepochsizeShort}(c), Minimax Regret Bound} 
Using the definition of the elimination rule, on $\cG_T$ we have $\widetilde{\Delta}_p^{m,k} \leq 8M \epsilon_{p-1}$. Directly summing over these terms for all epochs yields an exploration regret scaling with $\sum_{m,k} \sqrt{t_{m,k}}$, where $t_{m,k}$ roughly corresponds to the number of exploration rounds associated with edge $(m,k)$. This regret is maximized when all $t_{m,k}$ are equal, which leads to the sublinear regret bound of Theorem~\ref{thm:generalepochsizeShort}(c). See Appendix \ref{app:minmaxbound} for the rigorous statement and proof.


\section{Numerical Experiments}\label{sec:experiments}

We executed the following algorithms:\footnote{The source codes are included in the supplementary material.} 
M-ETC-Elim with $c=1$ and $c=2$,
GoT (the latest version in~\cite{got_arxiv}) with parameters\footnote{These parameters and the reward matrix $U_1$ are taken from the simulations section of~\cite{got_arxiv}.} $\delta=0,\eps = 0.01, c_1=500, c_2=c_3=6000$
and Selfish-UCB, a heuristic studied by \cite{emily_multiplayer,BoursierPerchet18} in the homogeneous setting which often performs surprisingly well despite the lack of theoretical evidence. In Selfish-UCB, each player runs the UCB1 algorithm of \cite{Aueral02} on the reward sequence $(R^m(t))_{t=1}^{\infty}$.\footnote{Note that this sequence is \emph{not} i.i.d.\ due to some observed zeros that are due to collisions.}
We experiment with Bernoulli rewards and the following reward matrices, whose entry $(m,k)$ gives the value of $\mu^m_k$:
\[
\tiny
U_1 = \begin{pmatrix}
0.1 & 0.05 & 0.9 \\
0.1 & 0.25 & 0.3 \\
0.4 & 0.2 & 0.8
\end{pmatrix},
U_2=\begin{pmatrix}
0.5& 		0.49 &0.39& 	0.29& 0.5\\
0.5& 		0.49 &0.39& 	0.29& 0.19\\
0.29& 	0.19 &0.5 	&	0.499& 0.39\\
0.29& 	0.49& 0.5 &		0.5  &  0.39\\
0.49& 	0.49& 0.49& 	0.49& 0.5
\end{pmatrix}.\]
Figure~\ref{expmain} reports the algorithms' {regrets} for various time horizons $T$, averaged over 100 independent replications. 
The first instance (matrix $U_1$, left plot) has a unique optimal matching and we observe that M-ETC-Elim has logarithmic regret (as promised by Theorem~\ref{thm:main}) and largely outperforms all competitors.
The second instance (matrix $U_2$, right plot) is more challenging, with more arms and players, two optimal matchings and several near-optimal matchings.
M-ETC-Elim with $c=1$ performs the best for large $T$ as well, though Selfish-UCB is also competitive.
Yet there is very little theoretical understanding of Selfish-UCB, and it fails badly on the other instance.
Appendix~\ref{sec:ExpesPlus} contains 
additional experiments corroborating our findings, where we also discuss practical  aspects  of implementing \NameSecond{}.

\section{Conclusion}
\label{sec:conclude}
We have presented
a practical algorithm for the heterogeneous multiplayer multi-armed bandit problem, which can be used in the presence of either unique or multiple optimal matchings and get a nearly logarithmic regret in both cases, as well as a sublinear regret in the worst case.
\NameSecond{} crucially relies on the assumption that the collision indicators are observed in each round. 
In future work, we aim to find algorithms with logarithmic regret in the setting when the players observe their rewards $R^m(t)$ only. So far, such algorithms have been proposed only in the homogeneous setting, see \citep{multiplayer_bandits_gabor_abbas,BoursierPerchet18}.     



\bibliography{biblioBandits}

\clearpage
\onecolumn
\appendix

\section{Description of the Initialization Procedure and Followers' Pseudocode}\label{sec:Init}

The pseudocode of the $\textnormal{\sc Init}(K,\delta_0)$ procedure, first introduced by \cite{BoursierPerchet18}, is presented in Algorithm~\ref{alg:init} for the sake of completeness. We now provide a proof of Lemma~\ref{lem:init}.

\begin{algorithm}[h]
	\DontPrintSemicolon
	\caption{{\sc Init}, the initialization algorithm\label{alg:init}}
	\KwIn{number of arms $K$, failure probability $\delta_0$}
	\KwOut{Ranking $R$, number of players $M$}
	\tcp{first, occupy a distinct arm using the musical chairs algorithm}
	$k \longleftarrow 0$\;
	\For(\tcp*[f]{rounds $1,\dots,T_0$}){$T_0 \coloneqq K \ln (K/\delta_0) $ rounds} {
		\eIf{$k=0$}
		{pull a uniformly random arm $i\in[K]$\;
			\lIf(\tcp*[f]{arm $k$ is occupied}){no collision occurred}{$k\longleftarrow i$}	
		}
		{pull arm $k$\;}
	}
	\tcp{next, learn $M$ and identify your ranking}
	$R \longleftarrow 1$\;
	$M \longleftarrow 1$\;
	\For(\tcp*[f]{rounds $T_0+1,\dots,T_0+2k-2 $}){$2k-2$ rounds}
	{pull arm $k$\;
		\If{collision occurred}
		{$R\longleftarrow R+1$\;{$M\longleftarrow M+1$}}
	}
	\For(\tcp*[f]{rounds $T_0+2k-1,\dots,T_0+K+k-2 $}){$i=1,2,\dots,K-k$}{pull arm $k+i$\;
		\If{collision occurred}
		{$M\longleftarrow M+1$}
	}
	\For(\tcp*[f]{rounds $T_0+K+k-1,\dots,T_0+2K-2 $}){$K-k$ rounds}{pull arm 1}
\end{algorithm}

Let $T_0\coloneqq K \ln(K/\delta_0)$.
During the first $T_0$ rounds, each player tries to occupy a distinct arm using the so-called musical chairs algorithm, first introduced in~\cite{musicalchair}:
she repeatedly pulls a random arm until she gets no collision, and then sticks to that arm.
We claim that after $T_0$ rounds, with probability $1-\delta_0$ all players have succeeded in occupying some arm.
Indeed, the probability that a given player $\mathcal A$, who has not occupied an arm so far, does not succeed in the next round is at most $1-1/K$, since there exists at least one arm that is not pulled in that round, and this arm is chosen by $\mathcal A$ with probability $1/K$.
Hence, the probability that $\mathcal A$ does not succeed in occupying an arm during these $T_0$ rounds is not more than
$$(1-1/K)^{T_0} < \exp(-T_0/K) = \delta_0/K \leq \delta_0/M,$$
and a union bound over the $M$ players proves the claim.

Once each player has occupied some arm, the next goal is to determine the number of players and their rankings.
This part of the procedure is deterministic.
The players' rankings will be determined by the indices of the arms they have occupied: a player with a smaller index will have a smaller ranking.
To implement this, a player that has occupied arm $k\in[K]$
will pull this arm for $2k-2$ more rounds (the waiting period), and will then sweep through the arms $k+1,k+2,\dots,K$, and can learn the number of players who have occupied arms in this range by counting the number of collisions she gets.
Moreover, she can learn the number of players occupying arms $1,\dots,k-1$ by counting the collisions during the waiting period;
see Algorithm~\ref{alg:init} for details.
The crucial observation to verify   the correctness of the algorithm is that two players occupying arms $k_1$ and $k_2$ will collide exactly once, and that happens at round $T_0 + k_1 + k_2 - 2$.

Next, we present the
pseudocode that the followers execute in \NameSecond{}.
Recall that \texttt{(comm.)} refers to a call to the communication protocol.

\label{sec:elimfollower}
\begin{procedure}[h]
	\DontPrintSemicolon
	\caption{FollowerAlgorithm(R,M) for the \NameSecond{} algorithm with parameter $c$}
	\KwIn{Ranking $R$, number of players $M$}
	\For{$p=1,2,\dots$} {
		Receive the value of $\operatorname{size}(\cC)$\tcp*[r]{(comm.)}
		\For{$i=1,2,\dots,\operatorname{size}(\cC)$} {
			Receive the arm assigned to this player in $\cC[i]$ \tcp*[r]{(comm.)}
		}
		Receive the communication arm of the leader and of this player\;
		\If{$size(\cC)=1$\tcp*[r]{(enter exploitation phase)}} {pull for the rest of the game the arm assigned to this player in the unique matching in $\cC$ \;}
		\For{$i=1,2,\dots,\operatorname{size}(\cC)$} {
			pull $2^{p^c}$ times the arm assigned to this player in the matching $\cC[i]$ \;
		}
		\For{$k=1,2,\dots,K$}{
			$\widehat{\mu}_k^R\longleftarrow$ empirically estimated  utility of arm $k$ if arm $k$ has been pulled in this epoch, 0 otherwise\;
			Truncate $\widehat{\mu}_k^R$ to
			$\widetilde{\mu}_k^R$ using the $\frac{p^c+1}{2}$ most significant bits\;
		}
		Send the values $\widetilde{\mu}_1^R,\widetilde{\mu}_2^R,\dots,\widetilde{\mu}_K^R$ to the leader \tcp*[r]{(comm.)}
	}
\end{procedure}


\section{Practical Considerations and Additional Experiments}\label{sec:ExpesPlus}

\subsection{Implementation Enhancements for M-ETC-Elim}\label{sec:tweaks}
In the implementation of M-ETC-Elim, the following enhancements significantly improve the regret  in practice (and have been used for the reported numerical experiments), but only by constant factors in theory, hence we have not included them in the analysis for the sake of brevity.

First, to estimate the means, the players are better off taking into account all pulls of the arms, rather than just the last epoch.
Note that after the exploration phase of epoch $p$, each candidate edge has been pulled $N_p \coloneqq \sum_{i=1}^{p} 2^{i^c}$ times.	Thus, with probability at least $1 - 2 \lg(T) / (MT)$, each edge has been estimated within additive error $\leq \eps'_p = \sqrt{\ln(M^2TK)/2N_p}$ by Hoeffding's inequality. The players then  truncate these estimates using $b \coloneqq \lceil - \lg (0.1 \eps'_p) \rceil$ bits,
adding up to $0.1 \eps'_p$ additive error due to quantization.
They then send these $b$ bits to the leader.
Now, the threshold for eliminating a matching would be
$2.2 M\eps'_p$ rather than
$4 M \times \sqrt{\ln(2M^2KT^2)/2^{1+p^c}}$
(compare with line~\ref{gapdefined} of the LeaderAlgorithm presented on page~\pageref{gapdefined}).

The second enhancement is to choose the set $\mathcal{C}$ of matchings to explore more carefully.
Say that a matching is \emph{good} if its estimated gap is at most $2.2 M\eps'_p$, and say an edge is \emph{candidate} (lies in $\cE$) if it is part of some good matching.
There are at most $MK$ candidate edges, and we need only estimate those in the next epoch.
Now, for each candidate edge, we can choose any good matching containing it, and add that to $\mathcal{C}$.
This guarantees that $|\mathcal{C}|\leq MK$, which gives the bound in Theorem~\ref{thm:main}.
But to reduce the size of $\mathcal{C}$ in practice, we do the following: initially, all edges are candidate. After  each exploration phase, we do the following: we mark all edges as \emph{uncovered}. For each candidate uncovered edge $e$, we compute the maximum matching $\pi'$ containing $e$ (using estimated means).
If this matching $\pi'$ has gap larger than $2.2 M \eps'_p$, then it is not good hence we remove $e$ from the set of candidate edges.
Otherwise, we add $\pi'$ to $\mathcal{C}$, and moreover, we mark all of its edges as \emph{covered}.
We then look at the next uncovered candidate edge, and continue similarly, until all candidate edges are covered.
This guarantees that all the candidate edges are explored, while the number of explored matchings could be much smaller than the number of candidate edges, which results in faster exploration and a smaller regret in practice.

To reduce the size of $\mathcal{C}$ even further, we do the following after each exploration phase:
first, find the maximum matching (using estimated means), add it to $\mathcal C$, mark all its edges as covered, and only then start looking for uncovered candidate edges as explained above.

\subsection{Other Reward Distributions}
In our model and analysis, we have assumed $Y^m_{k,t} \in [0,1]$ for simplicity (this is a standard assumption in online learning), but it is immediate to generalize the algorithm and its analysis to reward distributions bounded in any known interval via a linear transformation.
Also, we can adapt our algorithm and analysis to subgaussian distributions with mean lying in a known interval. A random variable $X$ is $\sigma$-subgaussian if for all $\lambda\in\R$ we have   $\bE[e^{\lambda (X-\bE X)}] \leq e^{\sigma^2 \lambda^2/2}$. This includes Gaussian distributions and distributions with bounded support. Suppose for simplicity that the means lie in $[0,1]$.
Then the algorithm need only change in two places:
first, when the followers are sending the estimated means to the leader, they must send 0 and 1 if the empirically estimated mean is $<0$ and $>1$, respectively.
Second, the definition of $\eps_p$ must be changed to
$\eps_p \coloneqq \sqrt{\sigma^2\ln(2/\delta)/2^{p^c-1}}$.
The only change in the analysis is that instead of using Hoeffding's inequality which requires a bounded distribution, one has to use 
a concentration inequality for sums of subgaussian distributions, see, e.g., \cite[Proposition~2.5]{wainwright2019high}.

We executed the same algorithms as in Section~\ref{sec:experiments} with the same reward matrices  but with Gaussian rewards with variance 0.05. The results are somewhat similar to the Bernoulli case and can be found in Figure~\ref{expapp}. 

\begin{figure}[h]
		\begin{center}
			\includegraphics[scale=0.82]{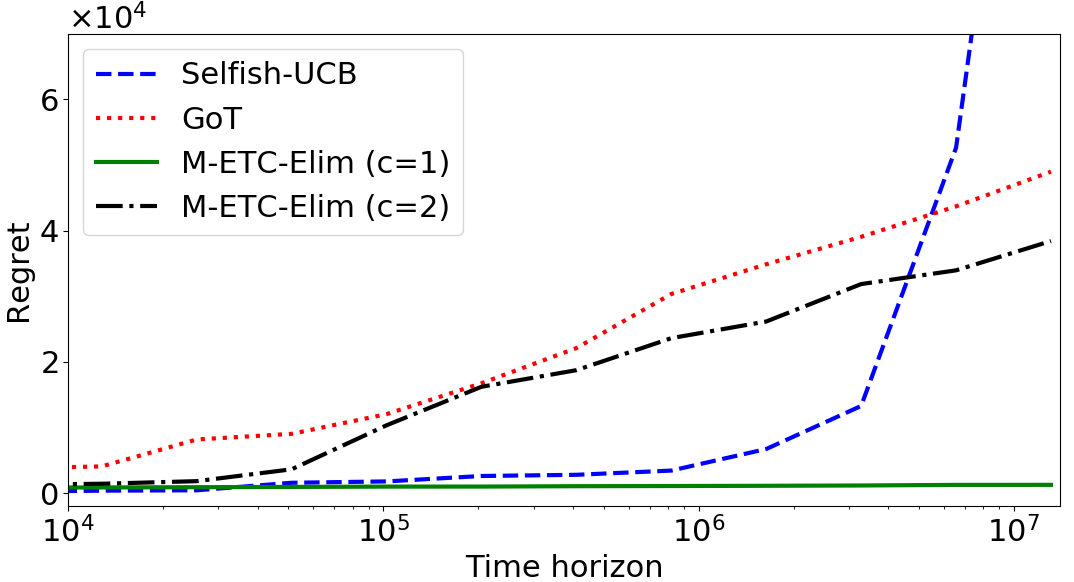}
			\includegraphics[scale=0.82]{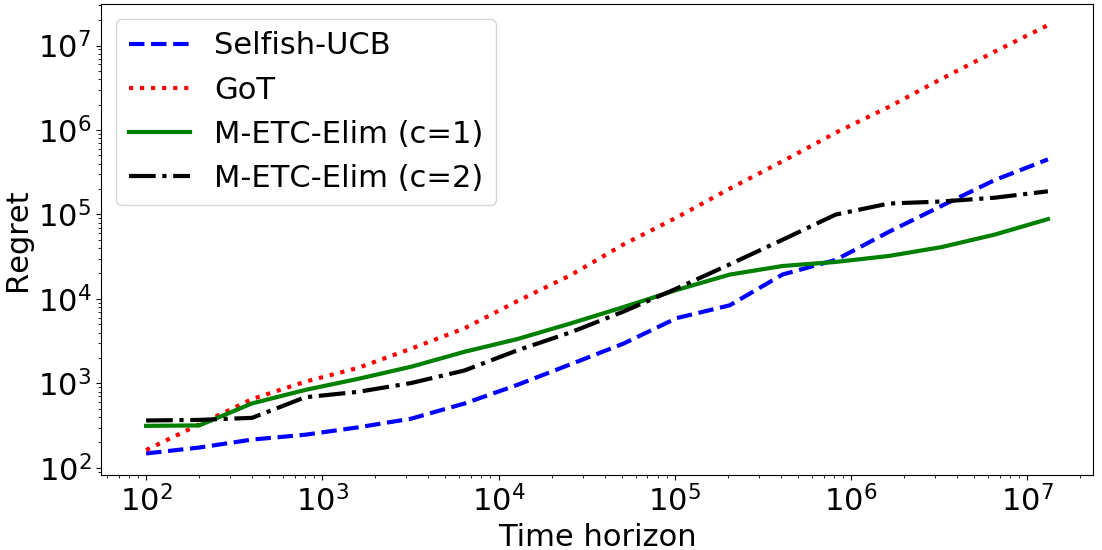}
		\end{center}
		\caption{Numerical comparison of M-ETC-Elim, GoT and Selfish-UCB on reward matrices $U_1$ (left)
			and $U_2$ (right) with Gaussian rewards and variance 0.05. The x-axis has logarithmic scale in both plots.
			The y-axis has logarithmic scale in the right plot.
		}
		\label{expapp}	
\end{figure}

The reason we performed these Gaussian experiments is to have a more fair comparison against GoT. Indeed, the numerical experiments of \cite{got_arxiv} rely on the same reward matrix $U_1$ and Gaussian rewards.

\section{Regret Analysis in the Presence of a Unique Maximum Matching}\label{app:uniquematching}
In Theorem~\ref{thm:cyclethroughmatchings} below we provide a refined analysis of \NameSecond{} with parameter $c=1$ if the maximum matching is unique, justifying the $O(\frac{KM^3}{\Delta}\ln(T))$ regret upper bound stated in Theorem~\ref{thm:generalepochsizeShort}(b). Its proof, given below, follows essentially the same line as the finite-time analysis given in Section~\ref{sec:ProofMain}, except for the last part.
Recall that $\ln(\cdot)$ denotes the natural logarithm and $\lg(\cdot)$ denotes logarithm in base 2.

\begin{theorem}\label{thm:cyclethroughmatchings} If the maximum matching is unique, for any $T>0$ the regret of the \NameSecond{} algorithm with parameter $c=1$ is upper bounded by
\begin{small}
	\begin{align*}
	&2+MK \ln (e^2 K^2 T)
	 +   3M^2 K  \lg(K) \lg \left( \frac{64M^2 \ln(2M^2KT^2)}{\Delta^2} \right) 
	+ M K \lg^2 \left( \frac{64M^2 \ln(2M^2KT^2)}{\Delta^2} \right)
	\\&+ \frac{4\sqrt{2}-2}{3-2\sqrt{2}} M^3K \lg (K)\sqrt{\ln(2M^2 K T^2)} + \frac{2\sqrt{2}-1}{\sqrt{2}-1}\!\!\!\!\!\!\!\!\!\! \mathlarger{\mathlarger{\sum_{(m,k) \in [M] \times [K]}}}\!\!\!\!\!\!\!\!\!\!
	\frac{64 M^2 \ln (2M^2KT^2)}{\Delta(\pi^{m,k})}. 
	\end{align*}
\end{small}
\end{theorem}

\begin{proof} The good event and the regret incurred during the initialization phase are the same as in the finite-time analysis given in Section~\ref{sec:ProofMain}. 
Recall the definition of $P$, which is
\(P(\pi) = \inf \{ p \in \N : 8 M \eps_{p} < \Delta(\pi)\}.\)
When there is a unique optimal matching, if the good event happens, the \NameSecond{} algorithm will eventually enter the exploitation phase, so $\hat p_T$ can be much smaller than the crude upper bound given by Lemma~\ref{lem:Tassumption}. Specifically, introducing $\pi'$ as the second maximum matching so that $\Delta(\pi') = \Delta$, we have, on the event $\cG_T$, 
\[\hat{p}_T \leq P(\pi') \leq \lg \left( \frac{64M^2 \ln(2M^2KT^2)}{\Delta^2} \right).\]

Plugging this bound in Lemma~\ref{lem:initcomm} yields that the regret incurred during communications is bounded by
\begin{small}\begin{align*}
3M^2 K\lg(K)\lg\left(\frac{64M^2 \ln(2M^2KT^2)}{\Delta^2}\right) + MK\lg^2\left(\frac{64M^2 \ln(2M^2KT^2)}{\Delta^2}\right) \\ + \frac{2M^3K\lg K}{\sqrt{2}-1}\sqrt{\ln(2/\delta)} + \frac{2\sqrt{2}}{3-2\sqrt{2}} M^2K \sqrt{\ln(2/\delta)}.
\end{align*}\end{small}
Also, for $c=1$ and any matching $\pi$,
the definition of $\eps_p$ in \eqref{def:deltaeps2new} gives
\begin{equation*}
P(\pi) \leq 
1 +  \lg \left( \frac{32 M^2 \ln (2M^2KT^2)}{\Delta(\pi)^2} \right).
\end{equation*}
In particular, $\Delta(\pi) 2^{P(\pi)} \leq \frac{64 M^2 \ln (2M^2KT^2)}{\Delta(\pi)}$. Using the same argument as in Section~\ref{sec:ProofMain}, the regret incurred during the exploration phases is bounded by
\begin{small}\begin{equation*}
\frac{2\sqrt{2}-1}{\sqrt{2}-1}\!\!\!\!\!\!\!\!\!\! \mathlarger{\mathlarger{\sum_{(m,k) \in [M] \times [K]}}}\!\!\!\!\!\!\!\!\!\!
	\frac{64 M^2 \ln (2M^2KT^2)}{\Delta(\pi^{m,k})}.\end{equation*}\end{small}Summing up the regret bounds for all phases proves Theorem~\ref{thm:cyclethroughmatchings}.\end{proof}
\section{Minimax Regret Analysis}\label{app:minmaxbound}
In Theorem~\ref{thm:minmaxbound} below we provide a minimax regret bound for \NameSecond{} with parameter $c=1$, justifying the $O\left(M^{\frac{3}{2}}\sqrt{KT\ln(T)}\right)$ regret upper bound stated in Theorem~\ref{thm:generalepochsizeShort}(c).

\begin{theorem}\label{thm:minmaxbound} For all $T$, the regret of the \NameSecond{} algorithm with parameter $c=1$ is upper bounded by
\begin{small}
	\begin{align*}
	&2+MK \ln (e^2 K^2 T)
	 +   3M^2 K  \lg(K) \lg \left(T \right) 
	+ M K \lg^2 \left( T \right)
	\\&+ \frac{4\sqrt{2}-2}{3-2\sqrt{2}} M^3K \lg (K)\sqrt{\ln(2M^2 K T^2)} + \frac{8}{\sqrt{2}-1} K^{\frac{1}{2}}M^{\frac{3}{2}} \sqrt{T\ln (2M^2KT^2)}. 
	\end{align*}
\end{small}
\end{theorem}
Note that the above regret bound is independent of the suboptimality gaps.

\begin{proof}
The good event and the regret incurred during the initialization phase are the same as in the finite-time analysis given in Section~\ref{sec:ProofMain}. Furthermore, using Lemma~\ref{lem:initcomm} stated therein and since $\hat p_T\leq\lg(T)$, the regret incurred during the communication phases is bounded by 
\[
3M^2 K  \lg(K) \lg \left(T \right) 
	+ M K \lg^2 \left( T \right)+ \frac{4\sqrt{2}-2}{3-2\sqrt{2}} M^3K \lg (K)\sqrt{\ln(2M^2 K T^2)}.
\]

We next bound the exploration regret.
Fix any edge $(m,k)$, and let $\widetilde{P}^{m,k}$ be the last epoch in which this edge is explored. 
If this edge belongs to an optimal matching, i.e., if $\pi^{m,k}$ is optimal, we instead define $\widetilde{P}^{m,k}$ as the last epoch in which the pulled matching $\widetilde{\pi}_p^{m,k}$ associated with $(m,k)$ is suboptimal. In either case, the contribution of the edge $(m,k)$ 
to the exploration regret
can be bounded by $\sum_{p=1}^{\widetilde{P}^{m,k}} \widetilde{\Delta}_p^{m,k} 2^{p}$.


Fix an epoch $p \leq \widetilde{P}^{m,k}$. 
Recall that $\cC_p$ contains at least one actual maximum matching, which we denote by $\pi^\star$.
Also, let $\tilde{\pi}_p^\star$ denote the maximum empirical matching right before the start of epoch $p$.
Since $(m,k)$ is candidate in epoch $p$, we have
%
\begin{eqnarray*}
\tilde{\Delta}_p^{m,k}& =& U^\star - U_{p-1}(\pi^\star) + U_{p-1}(\pi^\star) - U_{p-1}(\tilde{\pi}_p^{m,k}) + U_{p-1}(\tilde{\pi}_p^{m,k}) - U(\tilde{\pi}^{m,k}) \\
& \leq & ( U_\star - U_{p-1}(\pi^\star)) + ( U_{p-1}(\tilde{\pi}_p^\star) - U_{p-1}(\tilde{\pi}_p^{m,k}) + (U_{p-1}(\tilde{\pi}_p^{m,k}) - U(\tilde{\pi}_p^{m,k})) \\
& \leq & 2M\epsilon_{p-1} + 4M\epsilon_{p} + 2M\epsilon_{p-1}
\\& \leq &
8M \epsilon_{p-1} = 8M \sqrt{\frac{\ln(2/\delta)}{2^p}},
\end{eqnarray*}
so, the contribution of the edge $(m,k)$ to the exploration regret can further be bounded by
\begin{align*}
\sum_{p=1}^{\widetilde{P}^{m,k}} \widetilde{\Delta}_p^{m,k} 2^{p}  
\leq 8M \sqrt{\ln(2/\delta)}  \left({{\sum_{p=1}^{\widetilde{P}^{m,k}}} } \sqrt{2}^{p} \right)  
< \frac{8\sqrt{2} M \sqrt{\ln(2/\delta)}}{\sqrt{2}-1} \sqrt{2}^{\widetilde{P}^{m,k}}.
\end{align*}
To bound the total exploration regret, we need to sum this over all edges $(m,k)$.

Note that during each epoch $p=1,2,\dots,\tilde{P}_{m,k}$,
there are exactly $2^p$ exploration rounds associated with the edge $(m,k)$.
Since the total number of rounds is $T$, we find that
\[
\sum_{(m,k)\in[M]\times[K]} \sum_{p=1}^{\tilde{P}_{m,k}} 2^p \leq T,
\]
and in particular,
\[
\sum_{(m,k)\in[M]\times[K]} 2^{\tilde{P}_{m,k}}  \leq T,
\]
hence by the Cauchy-Schwarz inequality,
\[
\sum_{(m,k)\in[M]\times[K]} {\sqrt{2}}^{\tilde{P}_{m,k}}  = \sum_{(m,k)\in[M]\times[K]} \sqrt{{{2}}^{\tilde{P}_{m,k}} }
\leq \sqrt{MKT},
\]
so the total exploration regret can be bounded by
\[
\frac{8\sqrt{2} M \sqrt{\ln(2/\delta)}}{\sqrt{2}-1} 
\sum_{(m,k)\in[M]\times[K]} 
\sqrt{2}^{\widetilde{P}^{m,k}}
\leq
\frac{8\sqrt{2} M \sqrt{\ln(2/\delta)}}{\sqrt{2}-1}  \sqrt{MKT},
\]

completing the proof of Theorem~\ref{thm:minmaxbound}.\end{proof}

\section{Proofs of Auxiliary Lemmas for Theorems~\ref{thm:generalepochsize} and \ref{thm:cyclethroughmatchings}}
\label{app:elimproof}

\subsection{Proof of Lemma~\ref{lem:conc}}\label{sec:conc}
We recall Hoeffding's inequality. 

\begin{proposition}[Hoeffding's inequality \protect{\cite[Theorem~2]{hoeffding}}]
	Let $X_1,\dots,X_n$ be independent random variables taking values in $[0,1]$.
	Then for any $t\geq0$ we have
	\[\pr{\left|\frac{1}{n}\sum X_i - \ex{\frac{1}{n}\sum X_i}\right| > t} < 2\exp(-2nt^2).\]
\end{proposition}

Recall the definition of the good event 
\[\cG_T = \Big\{ \textsc{Init}(K,1/KT) \text{ is successful and }  \forall p \leq \hat{p}_T, \forall \pi \in \cC_{p+1}, |\tilde{U}_p(\pi) - U(\pi)|\leq 2M\epsilon_p\Big\}.\]
and recall that
$\eps_p \coloneqq \sqrt{\ln (2/\delta) / 2^{p^c+1} }$ and $\delta = 1/M2KT^2$.
Let $\cH$ be the event that $\textsc{Init}(K,1/KT) $ is successful for all players. Then,
\begin{eqnarray*}
\bP\left(\cG_T^c\right) &\leq & \bP\left(\cH^c\right) + \bP\left(\cH \textnormal{ happens and } \exists p \leq \hat{p}_T,\exists \pi \in \cM \text{ with candidate edges} \textnormal{ such that } |\tilde{U}_p(\pi) - U(\pi)|> 2M\epsilon_p\right) \\
& \leq & \frac{1}{KT} + \bP\left(\cH \textnormal{ happens and } \exists p \leq \lg(T), \exists \pi \in \cM  \text{ with candidate edges such that} |\tilde{U}_p(\pi) - U(\pi)| > 2M\epsilon_p\right),
\end{eqnarray*}
where we have used that $\hat p_T \leq \lg(T)$ deterministically.

Fix an epoch $p$ and a candidate edge $(m,k)$.
We denote by $\widehat\mu_{k}^m(p)$ the estimated mean of arm $k$ for player $m$ at the end of epoch $p$ and by $\widetilde\mu_{k}^m(p)$ the truncated estimated mean sent to the leader by this player at the end of epoch $p$. 

By Hoeffding's inequality and since this estimated mean is based on at least $2^{p^c}$ pulls, we have
\[
\pr{ | \widehat\mu_{k}^m(p) - \mu_{k}^m | > \eps_p} < \delta.
\]
The value $\tilde\mu_{k}^m(p)\in[0,1]$ which is sent to the leader uses the $(p^c+1)/2$ most significant bits. The truncation error is thus at most $2^{-(p^c+1)/2} < \eps_p$, hence we have 
\[
\pr{ | \tilde\mu_{k}^m(p) - \mu_{k}^m | > 2\eps_p} < \delta.
\]
Given the event $\cH$ that the initialization is successful, the quantity $\tilde{U}_p(\pi)$ is a sum of $M$ values $\widetilde\mu_{k}^m(p)$ for $M$ different edges $(m,k)\in[M]\times[K]$. Hence, we have
\begin{align*}
&\bP\left( \cH \textnormal{ happens and } \exists \pi \in \cM \text{ with candidate edges such that } |\tilde{U}_p(\pi) - U(\pi)|> 2M\eps_p | \right)\\ 
&\hspace{1cm}\leq \bP\left(\exists \text{ candidate edge } (m,k) \text{ such that } |\widetilde{\mu}_k^m(p) - \mu_k^m|> 2\eps_p\right) \leq KM\delta.
\end{align*}
Finally, a union bound on $p$ yields 
\begin{eqnarray*}
\bP\left(\cG_T^c\right)
& \leq & \frac{1}{KT} + \lg(T) KM\delta \leq \frac{1}{MT} + \frac{1}{MT},
\end{eqnarray*}
completing the proof of Lemma~\ref{lem:conc}

\subsection{Proof of Lemma~\ref{lem:initcomm}}
	
	For each epoch $p$, the leader first communicates to each player the list of candidate matchings.
	There can be up to $MK$ candidate matchings, and for each of them the leader communicates to the player the arm she has to pull (there is no need to communicate to her the whole matching) which requires $ \lg K$ bits, and there are a total of $M$ players, so this takes at most $M^2 K \lg (K)$ many rounds.\footnote{Strictly speaking, the leader also sends her communication arm and the size of the list she is sending, but there are  at most $MK-M+1$ candidate matchings, as the best one is repeated $M$ times. So, this communication still takes at most $M^2 K \lg K$ many rounds.}
	
	At the end of the epoch, each player sends the leader the empirical estimates for the arms she has pulled, which requires at most $M K(1+ p^c)/2$ many rounds. 
	As players use the best estimated matching as communication arms for the communication phases, a single communication round incurs regret at most $2 + 2M\epsilon_{p-1}$, since the gap between the best estimated matching of the previous phase and the best matching is at most $2M \epsilon_{p-1}$ conditionally to $\cG_T$ (we define $\epsilon_0\coloneqq \sqrt{\frac{\ln(2/\delta)}{2}} \geq \frac{1}{2}$). The first term is for the two players colliding, while the term $2M\epsilon_{p-1}$ is due to the other players who are pulling the best estimated matching instead of the real best one.
	With $\hat{p}_T$ denoting the number of epochs before the (possible) start of the exploitation, the total regret due to  communication phases can be bounded by
	\begin{small}
\begin{align*}
	R_c & \leq  \sum_{p=1}^{\hat p_T}  
	\left(
	2M^2 K  \lg (K)+
	M K (1+p^c)
	\right)(1 + M\epsilon_{p-1}) \\
	& \leq  3M^2K \lg (K) \hat{p}_T + M K (\hat{p}_T)^{c+1} + M^2K \sum_{p=1}^{\hat p_T}  
	\left( 2M \lg (K)+ (1+p^c) \right)
	\epsilon_{p-1}.
	\end{align*}
	\end{small}

We now bound the sum as:
\begin{small}
\begin{align*}
\sum_{p=1}^{\hat p_T}  
	\left(2 M \lg (K)+ (1+p^c) \right)
	\epsilon_{p-1} & = 2M \lg (K)\sqrt{\ln(2/\delta)} \sum_{p=0}^{\hat p_T-1}  \frac{1}{\sqrt{2}^{1+p^c}} + \sqrt{\ln(2/\delta)} \sum_{p=0}^{\hat p_T-1}  \frac{1+(p+1)^c}{\sqrt{2}^{1+p^c}}\\
	& \leq 2M \lg (K)\sqrt{\ln(2/\delta)}\sum_{n=1}^{\infty}  \frac{1}{\sqrt{2}^{n}} + \sqrt{\ln(2/\delta)} \sum_{n=1}^{\infty}  \frac{n2^c}{\sqrt{2}^{n}} \\
	& \leq 2M \lg (K)\sqrt{\ln(2/\delta)} \frac{1}{\sqrt{2} -1} + \sqrt{\ln(2/\delta)}   \frac{2^c\sqrt{2}}{(\sqrt{2}-1)^2},
\end{align*}\end{small}completing the proof of Lemma~\ref{lem:initcomm}.

\subsection{Proof of Lemma~\ref{lem:Tassumption}}\label{proof:Tassumption}
The assumption $T\geq \exp(2^{\frac{c^c}{\ln^c(1+\frac{1}{2c})}})$ gives $\lg(\ln T)^{1/c} \geq \frac{c}{\ln(1+1/2c)}$.
In particular, $(\lg T)^{1/c} \geq {c}$.
	We will also use the inequality
	\begin{equation}\label{cx}
	(x+1)^c \leq e^{c/x} x^c,
	\end{equation}
	which holds for all positive $x$, since
	$(x+1)^c/x^c=(1+1/x)^c\leq \exp(1/x)^c = \exp(c/x)$.

Using a crude upper bound on the number of epochs that can fit within $T$ rounds, we get  $\hat{p}_T \leq 1 + (\lg T)^{1/c}$. As $(\lg T)^{1/c} \geq {c} \geq 1$ we have $\hat p_T \leq 2 (\lg T)^{1/c}$. Also~\eqref{cx} gives
$(\hat p_T)^{c}
\leq e \lg T$.

Also, $2\lg(\ln(T)) \geq 2c^c \geq 2^{c}$. It remains to show the first inequality of Lemma~\ref{lem:Tassumption}.

Straightforward calculations using the definition of $\eps_p$ in \eqref{def:deltaeps2new} give
\[
P(\pi) \leq 
1 + 
L(\pi)^{1/c},
\textnormal{ where } 
L(\pi)\coloneqq \lg \left( \frac{32 M^2 \ln (2M^2KT^2)}{\Delta(\pi)^2} \right).\]
We claim that we have
\begin{equation}\label{Ppi}
P(\pi)^c \leq \left(1+\frac{1}{2c}\right) L(\pi).
\end{equation}

Indeed, since $\Delta(\pi)\leq M$,  we have $L(\pi)^{1/c}> (\lg \ln T)^{1/c} \geq \frac{c}{\ln(1+1/2c)}$ and so \eqref{cx} with $x=L(\pi)^{1/c}$ gives~\eqref{Ppi}. 
Hence,
\begin{equation} \label{DeltaPpi}
\Delta(\pi) 2^{P(\pi)^c}
\leq \Delta(\pi) \left( \frac{32 M^2 \ln (2M^2 KT^2) }{\Delta(\pi)^2} \right)^{1+1/2c}
\leq \left(\frac{32 M^2 \ln(2M^2KT^2)}{\Delta(\pi)} \right)^{1+1/c},
\end{equation}
completing the proof of Lemma~\ref{lem:Tassumption}.

\subsection{Proof of Lemma~\ref{lemma:exploelimm}} \label{proof:exploelimm}
For brevity we define, for this proof, $\Delta \coloneqq \Delta(\pi^{m,k})$, $P \coloneqq P(\pi^{m,k})$ and $\Delta_p \coloneqq \widetilde{\Delta}_p^{m,k}$. First, $\Delta > 8 M \epsilon_P$ by definition of $P$. Also, $\Delta_p \leq 8 M \epsilon_{p-1}$ for any $p \leq P-1$, otherwise the edge $(m,k)$ would have been eliminated before epoch $p$. It then holds 
\begin{equation}
\label{eq:deltap}
\Delta_p \leq \frac{\epsilon_{p-1}}{\epsilon_P} \Delta = \sqrt{2}^{P^c - (p-1)^c} \Delta.
\end{equation}

It comes from the convexity of $x \mapsto x^c$ that $(p+1)^c + (p-1)^c - 2p^c \geq 0$, and thus 
$$P^c + (p-1)^c - 2p^c  \geq P^c - (p+1)^c \geq P - (p+1).$$ It then follows \[p^c + \frac{P^c - (p-1)^c}{2} \leq P^c + \frac{p+1 - P}{2}.\] Plugging this in \eqref{eq:deltap} gives
\[
2^{p^c} \Delta_p \leq \frac{2^{P^c}}{\sqrt{2}^{P-(p+1)}} \Delta,
\]
completing the proof of Lemma~\ref{lemma:exploelimm}.
\end{document}